\documentclass[letterpaper, 10 pt, conference]{ieeeconf}  

\IEEEoverridecommandlockouts                              
\usepackage{cite}
\usepackage{amsmath,amssymb,amsfonts}
\usepackage{graphicx}
\usepackage{tikz}
\usetikzlibrary{arrows.meta,calc,positioning}
\usepackage{hyperref}
\hypersetup{
  hidelinks,
  pdftitle={Adaptive Prescribed Performance Trajectory Tracking for Uncertain Autonomous Underwater Vehicles With Input Saturation},
  pdfauthor={Athanasios K. Gkesoulis; Christos Vlachos; George K. Fourlas; Charalampos P. Bechlioulis; George C. Karras},
  pdfsubject={Adaptive prescribed performance control of underwater vehicles with input saturation},
  pdfkeywords={Underwater robotics, prescribed performance control, input saturation, trajectory tracking}
}

\newtheorem{theorem}{Theorem}

\newtheorem{assumption}{Assumption}

\newtheorem{remark}{Remark}

\title{\LARGE \bf
Adaptive Prescribed Performance Trajectory Tracking for Uncertain Autonomous Underwater Vehicles 
With Input Saturation
}

\author{Athanasios K. Gkesoulis$^{1}$, Christos Vlachos$^{1,2}$, George K. Fourlas$^{3}$,\\
Charalampos P. Bechlioulis$^{1,2}$, and George C. Karras$^{1,3}$
\thanks{$^{1}$The authors are with the Robotics Institute, Athena Research Center, Marousi, Greece (e-mails: \{\texttt{a.gkesoulis@athenarc.gr}, \texttt{xristos.blaxos3@gmail.com}, \texttt{chmpechl@upatras.gr}, \texttt{gkarras@uth.gr}\}).}
\thanks{$^{2}$The authors are with the Division of Systems and Automatic Control, Department of Electrical and Computer Engineering, University of Patras, Patras, Greece.}
\thanks{$^{3}$The authors are with the Department of Informatics and Telecommunications, School of Sciences, University of Thessaly, Lamia, Greece (e-mail: \texttt{gfourlas@uth.gr}).}}

\makeatletter
\def\ps@preprintfirstpage{
\def\@oddhead{}
\def\@evenhead{}
\def\@oddfoot{\parbox[b]{\textwidth}{\normalfont\footnotesize This work has been submitted to the IEEE for possible publication. Copyright may be transferred without notice, after which this version may no longer be accessible.}}
\let\@evenfoot\@oddfoot
}
\makeatother

\begin{document}

\maketitle
\thispagestyle{preprintfirstpage}
\pagestyle{empty}

\begin{abstract}
Accurate trajectory tracking is essential for autonomous underwater vehicles operating with uncertain hydrodynamics and limited thruster authority. This paper develops a low-complexity adaptive prescribed performance controller for uncertain underwater vehicle dynamics with coupled and saturated inputs. The proposed method adapts the performance functions directly through the mismatch between the requested and saturated control inputs, enlarging the performance funnels only when the available actuation cannot meet the nominal specifications. A coupled invariant-set analysis establishes boundedness of all closed-loop signals, prescribed bounds on the filtered tracking errors, and explicit feasibility conditions. Comparative simulations on a BlueROV2 underwater vehicle show that the adaptive design maintains prescribed performance during a figure-eight maneuver, whereas conventional prescribed performance control loses feasibility under input limits. A real-time pool experiment further demonstrates stable trajectory tracking and the practical implementation of the proposed method on an underwater robotic platform.
\end{abstract}

\section{INTRODUCTION}
\label{sec:introduction}

Autonomous underwater vehicles (AUVs) are increasingly expected to execute inspection, mapping, intervention, and environmental monitoring missions under limited actuation authority. In these operations, trajectory tracking accuracy is a robotic mission requirement: large transient deviations may compromise sensor coverage, vehicle safety, or the quality of collected data. Underwater vehicles also operate with uncertain hydrodynamics, restoring forces, disturbances, and hard thruster limits whose effects propagate through control allocation and coupled body motion. Experimental prescribed performance tracking of underwater robots \cite{BechlioulisKarras2017} and open-water funnel control under input saturation \cite{lapandic2024kinodynamic} reinforce the need to jointly address tracking performance and actuation limits. 

Prescribed performance control (PPC) imposes transient and steady state tracking specifications through designer-selected performance envelopes \cite{4639441}, while funnel control achieves the same following a different design \cite{ilchmann2002tracking}. Barrier-like transformations of normalized errors can enforce convergence rate, overshoot, and residual error requirements without exact model identification. Actuator saturation, however, creates a feasibility conflict whenever a fixed performance envelope demands more control authority than the vehicle can deliver, particularly during aggressive maneuvers or in the presence of strong currents and other disturbances.

\begin{figure}[!t]
\centering
\includegraphics[width=0.70\columnwidth]{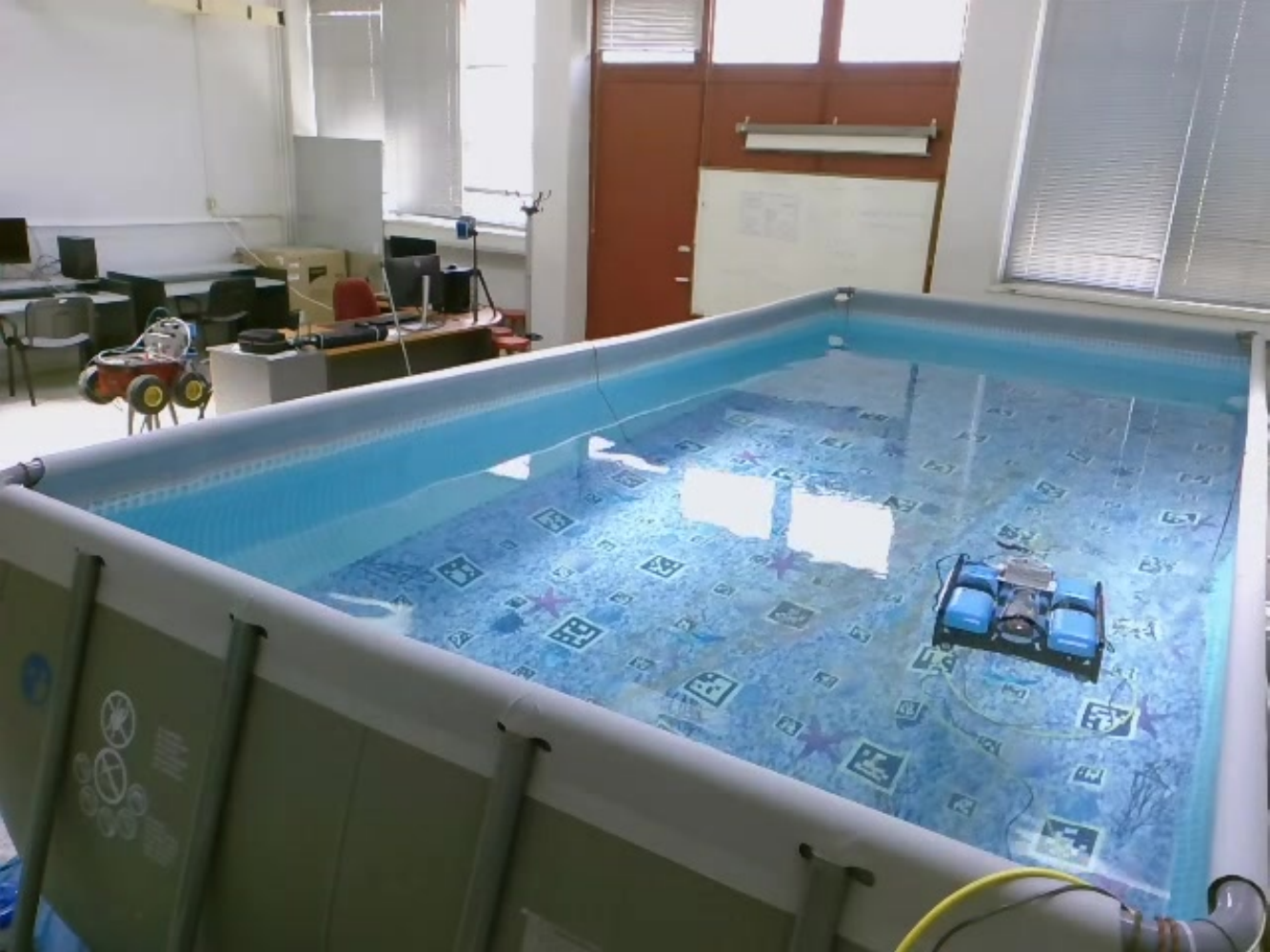}
\caption{BlueROV2 experimental platform operating above the mapped ArUco fiducials in the test tank.}
\label{fig:bluerov2_platform}
\end{figure}

Input constraints have been studied through adaptive, robust, neuroadaptive, finite-time, and sliding-mode control \cite{karason1993adaptive,CHEN2011452,chen2020robust,wang2021neuroadaptive,shao2021adaptive}. Saturation-aware PPC and funnel designs include adaptive-gain and MIMO funnel controllers \cite{ilchmann2004input,hu2022funnel,hopfe2010funnel}, bang-bang designs \cite{liberzon2010bang,liberzon2013bang}, disturbance-regularity and bounded-gain methods \cite{wang2019improved,ji2022saturation,ji2022saturation2}, and approximation-free constrained PPC \cite{10004950}. Flexible envelopes, switching mechanisms, auxiliary dynamics, observers, and neural or fuzzy approximators have also been employed \cite{YONG2020109268,trakas2023robust,10693631,BIKAS2024111663,li2018adaptive,cheng2019neural}. Recent underwater-vehicle work combines prescribed performance with dynamic compensation, neural approximation, filtering, and velocity observation to handle saturation and uncertainty \cite{LiangYuLi2025}. Although effective, the use of such structures increases implementation complexity on robotic platforms that usually operate with limited computational capabilities.

Among low-complexity alternatives, an adjustable outer funnel is employed in \cite{berger2024input}, adaptive performance bounds are adopted in \cite{10693631,11121659}, and desired-motion modification is considered in \cite{FotiadisRovithakisTACON2024,11187072} to recover feasibility. The MIMO adaptive performance design of \cite{10693631} exploits the available control authority across the input channels. However, its adaptation mechanism contains the saturation mismatch divided by the normalized error $|\xi_i|<1$, together with user-selected cross-channel coupling parameters. Consequently, its saturation-induced relaxation contains a term of order $\mathcal O(|\Delta\tau_i|/|\xi_i|)$, where $\Delta\tau_i$ denotes the mismatch between the saturated and required unsaturated input. The scalar design of \cite{11121659} instead adapts the performance envelope directly through $|\Delta\tau_i|$, avoids divisions by the normalized error, and yields an $\mathcal O(|\Delta\tau_i|)$ degradation bound, but it does not address coupled MIMO dynamics. Reference modification methods preserve the nominal performance funnel by altering the commanded motion. In particular, \cite{gkesoulis2026mimocoupled} treats uncertain MIMO systems with coupled and constrained inputs and obtains a mismatch-dependent tracking bound through reference adaptation rather than performance adaptation. In contrast, \cite{FotiadisRovithakisTACON2024} does not consider the state-dependent cross-channel input coupling induced by the inverse inertia matrix. Therefore, to the best of the authors' knowledge, no low-complexity adaptive performance result has yet considered direct mismatch-proportional adaptation of the performance bounds and an explicit treatment of the coupled input action of uncertain Euler--Lagrange underwater vehicles. This combination is challenging because the inverse inertia matrix maps each saturated generalized input into multiple acceleration channels. Saturation in one channel can thus affect all filtered errors, invalidating the decoupled invariant-set arguments employed in the scalar case \cite{11121659}.

Motivated by the aforementioned discussion, we develop an adaptive PPC framework for uncertain AUV Euler--Lagrange dynamics with coupled and saturated inputs. Each performance function is driven directly by the corresponding saturation mismatch and recovers its prescribed form exponentially when saturation is inactive, while the desired trajectory remains unchanged. To overcome the coupling challenge, we construct a coupled invariant set and a saturation-aware auxiliary Lyapunov function whose cancellation structure exploits the positive definite symmetric part of the inverse inertia matrix. This establishes boundedness of all closed-loop signals, explicit feasibility conditions, and a tracking degradation bound that is linear in the saturation mismatch, without the normalized-error denominator appearing in \cite{10693631}.

The main contributions are summarized as follows:
\begin{itemize}
\item We propose a low-complexity adaptive prescribed performance trajectory tracking controller for uncertain AUV Euler--Lagrange dynamics with saturated generalized inputs and coupled error channels, without reference modification, dedicated controller-side observers, regressors, neural networks, or fuzzy approximators.
\item We establish boundedness and adaptive prescribed performance through a coupled invariant-set argument and an auxiliary Lyapunov function tailored to the saturation mismatch.
\item We derive explicit sufficient feasibility conditions linking the required saturation levels to the performance envelopes, controller gains, uncertain dynamics, and the positive definite symmetric part of the inverse inertia matrix, together with a mismatch-dependent tracking degradation bound.
\item We validate the design through comparative BlueROV2 simulations and a real-time pool experiment.
\item The analysis explicitly accommodates coupled MIMO dynamics and constrained generalized inputs while retaining the commanded trajectory and adapting only the performance functions.
\end{itemize}

The rest of the paper is organized as follows. Section~\ref{sec:problem_formulation} states the AUV model, assumptions, and control objective. Section~\ref{Section:Main Results} presents the adaptive PPC design and the stability and performance analysis. Section~\ref{section:SIMULATION} reports comparative simulations and experimental validation, and Section~\ref{sec:conclusion} concludes the paper.

\section{Problem Formulation}
\label{sec:problem_formulation}

\subsection{System Dynamics}
\label{subsec:system_dynamics}

Consider an underwater vehicle described by the generalized Euler--Lagrange dynamics
\begin{align}
M(\eta,\dot\eta)\ddot\eta
+
C(\eta,\dot\eta)\dot\eta
+
N(t,\eta,\dot\eta)
=
\operatorname{sat}(\tau),
\label{eq:agent dynamics}
\end{align}
where $\eta:=[\eta_1,\ldots,\eta_m]^T\in\mathbb R^m$ denotes the generalized position vector, $\dot\eta\in\mathbb R^m$ denotes the generalized velocity vector, $M(\eta,\dot\eta)\in\mathbb R^{m\times m}$ is the inertia matrix, $C(\eta,\dot\eta)\dot\eta$ contains Coriolis and centripetal terms, and $N(t,\eta,\dot\eta)\in\mathbb R^m$ collects hydrodynamic damping, restoring forces, bounded external disturbances, and unmodeled dynamics. The vector $\tau:=[\tau_1,\ldots,\tau_m]^T\in\mathbb R^m$ is the requested generalized control input. The input is subject to componentwise saturation, namely
\begin{equation}
\operatorname{sat}(\tau_i)
=
\begin{cases}
\tau_i, & |\tau_i|<\tau_{i,\max},\\
\tau_{i,\max}\operatorname{sgn}(\tau_i), & |\tau_i|\ge \tau_{i,\max},
\end{cases}
\quad i=1,\ldots,m,
\label{eq:saturation}
\end{equation}
where $\tau_{i,\max}>0$ are the saturation limits. The class of underwater-vehicle dynamics considered is experimentally represented by the BlueROV2 platform shown in Fig.~\ref{fig:bluerov2_platform}.

We impose the following assumptions.

\begin{assumption}
\label{Assumption: regularity}
The matrices $M(\eta,\dot\eta)$, $C(\eta,\dot\eta)$, and the vector $N(t,\eta,\dot\eta)$ are continuous with respect to time and locally Lipschitz with respect to their remaining arguments. Moreover, $N(t,\eta,\dot\eta)$ is uniformly bounded with respect to time on compact subsets of $\mathbb R^m\times\mathbb R^m$.
\end{assumption}

\begin{assumption}
\label{Assumption: M symmetric part positive definite}
The symmetric part of the inertia matrix $M_s(\eta,\dot\eta):=[M(\eta,\dot\eta)+M^T(\eta,\dot\eta)]/2$ is positive definite for all $(\eta,\dot\eta)\in\mathbb R^m\times\mathbb R^m$.
\end{assumption}

\begin{assumption}
\label{Assumption: measurable states}
The generalized position $\eta$ and velocity $\dot\eta$ are available for measurement.
\end{assumption}

\begin{assumption}
\label{Assumption: Reference Bounded}
The desired trajectory $\eta_d(t):=[\eta_{d,1}(t),\ldots,\eta_{d,m}(t)]^T$ is twice continuously differentiable and the signals $\eta_d(t)$, $\dot\eta_d(t)$, and $\ddot\eta_d(t)$ are uniformly bounded for all $t\ge0$. The corresponding bounds are not required in the controller.
\end{assumption}

\begin{remark}
We assume that saturation is imposed on the generalized force vector after allocation. If saturation is imposed at the thruster level, then the allocation map must be incorporated into the input matrix and the corresponding positive definiteness condition must be rechecked.
\end{remark}

\subsection{Control Problem}
\label{subsec:control_problem}

Let $e(t):=\eta(t)-\eta_d(t)$ denote the position tracking error. For each $i=1,\ldots,m$, define the filtered tracking error $s_i(t):=\dot e_i(t)+\alpha_i e_i(t)$, with $\alpha_i>0$.
In vector form,
\begin{equation}
s(t)=\dot e(t)+A e(t),
\label{eq:filtered error vector}
\end{equation}
where $A:=\operatorname{diag}(\alpha_1,\ldots,\alpha_m)$. The control objective is to design a continuous state feedback control law such that all closed loop signals remain bounded and the filtered error $s_i(t)$ evolves within adaptive prescribed performance bounds. In contrast to standard prescribed performance control, the performance bounds are not fixed functions of time. Instead, they are dynamically adjusted according to the actual mismatch between the requested and saturated input. Thus, the performance specification is relaxed only when saturation occurs and recovers its nominal value when the requested input becomes feasible. The controller should not employ any knowledge of $M$, $C$, or $N$, and should not require dedicated controller-side observers, regressors, neural networks, fuzzy approximators, or other approximation structures.

\section{Main Results}
\label{Section:Main Results}

\subsection{Control Design}
\label{subsection:Control Design}

The proposed control design consists of the following steps.

\textit{Step 1:} For each $i=1,\ldots,m$, introduce the adaptive performance function $\rho_i(t)$ generated by
\begin{equation}
\dot\rho_i(t)
=
-\beta_i\left(\rho_i(t)-\rho_{i,\infty}\right)
+
\ell_i|\Delta\tau_i(t)|,
\label{eq:adaptive performance function}
\end{equation}
where $\beta_i>0$, $\ell_i>0$, and $\rho_{i,\infty}>0$ are design constants, and the saturation mismatch is defined as $\Delta\tau(t):=\tau(t)-\operatorname{sat}(\tau(t))$.
The initial values are selected such that
\begin{equation}
\rho_i(0)>\max\{|s_i(0)|,\rho_{i,\infty}\},
\qquad i=1,\ldots,m.
\label{eq:rho initial condition}
\end{equation}
Here, $\rho_{i,\infty}$ denotes the nominal steady-state funnel width, $\beta_i$ determines the exponential recovery rate toward it, $\ell_i$ sets the sensitivity of the funnel enlargement to the saturation mismatch, and $\rho_i(0)$ specifies the initial funnel width and is selected to contain $s_i(0)$.

\textit{Step 2:} Define the normalized filtered error
\begin{equation}
\xi_i(t):=\frac{s_i(t)}{\rho_i(t)},
\qquad i=1,\ldots,m.
\label{eq:xi scalar definition}
\end{equation}
Equivalently, $\xi(t)=R^{-1}(t)s(t)$, where $R(t):=\operatorname{diag}(\rho_1(t),\ldots,\rho_m(t))$.

\textit{Step 3:} Let $T:(-1,1)\to\mathbb R$ be an odd, continuously differentiable, strictly increasing function satisfying $\lim_{\xi\to1^-}T(\xi)=+\infty$, $\lim_{\xi\to-1^+}T(\xi)=-\infty$, and $T(0)=0$. For example, one may choose $T(\xi)=\ln[(1+\xi)/(1-\xi)]$. Define $\mathbf T(\xi):=[T(\xi_1),\ldots,T(\xi_m)]^T$ and $K:=\operatorname{diag}(k_1,\ldots,k_m)$, with $k_i>0$.
The requested control input is
\begin{equation}
\tau(t)=-K\mathbf T(\xi(t)).
\label{eq:control law vector form}
\end{equation}

\begin{remark}
The control law \eqref{eq:adaptive performance function}--\eqref{eq:control law vector form} does not modify the desired trajectory and does not introduce a reference modification signal. Instead, the prescribed performance functions are adapted directly through the input mismatch $|\Delta\tau_i|$. When saturation is inactive, $\Delta\tau_i=0$, and each $\rho_i(t)$ converges exponentially to $\rho_{i,\infty}$. When saturation occurs, the corresponding performance function relaxes smoothly.
\end{remark}

The resulting closed-loop signal flow is summarized in Fig.~\ref{fig:closed_loop_architecture}.

\begin{figure*}[t]
\centering
\begin{tikzpicture}[
  >=Latex,
  font=\footnotesize,
  node distance=3.5mm,
  block/.style={draw, align=center, minimum height=11mm, fill=gray!8},
  auxiliary/.style={draw, align=center, minimum height=10mm, fill=gray!15},
  line/.style={->, semithick}
]
\node[block, text width=19mm] (reference)
  {Desired motion\\$\eta_d,\dot\eta_d$};
\node[block, text width=27mm, right=of reference] (errors)
  {Tracking and filtering\\$e=\eta-\eta_d$\\$s=\dot e+Ae$};
\node[block, text width=23mm, right=of errors] (normalization)
  {Adaptive normalization\\$\xi=R^{-1}s$};
\node[block, text width=25mm, right=of normalization] (controller)
  {Transformed control\\$\tau=-K\mathbf T(\xi)$};
\node[block, text width=20mm, right=of controller] (saturation)
  {Input saturation\\$\operatorname{sat}(\tau)$};
\node[block, text width=23mm, right=15mm of saturation] (plant)
  {Uncertain coupled\\AUV dynamics};

\node[auxiliary, text width=43mm, below=13mm of normalization] (adaptation)
  {Performance adaptation\\$\dot\rho_i=-\beta_i(\rho_i-\rho_{i,\infty})+\ell_i|\Delta\tau_i|$};
\node[auxiliary, text width=38mm, below=13mm of saturation] (mismatch)
  {Input mismatch\\$\Delta\tau=\tau-\operatorname{sat}(\tau)$};

\draw[line] (reference) -- (errors);
\draw[line] (errors) -- node[above] {$s$} (normalization);
\draw[line] (normalization) -- (controller);
\draw[line] (controller) -- node[above] {$\tau$} (saturation);
\draw[line] (saturation) -- node[above=1mm,font=\scriptsize] {$\operatorname{sat}(\tau)$} (plant);
\coordinate (feedbackright) at ($(plant.south)+(0,-28mm)$);
\coordinate (feedbackleft) at (errors.south |- feedbackright);
\draw[line] (plant.south) -- (feedbackright) -- (feedbackleft) -- (errors.south);
\node[below=1mm,fill=white,inner sep=1pt] at ($(feedbackleft)!0.5!(feedbackright)$)
  {$\eta,\ \dot\eta$};
\draw[line] (controller.south) |- ($(mismatch.north west)+(0,-2mm)$);
\draw[line] (saturation.south) -- (mismatch.north);
\draw[line] (mismatch.west) -- (adaptation.east);
\draw[line] (adaptation.north) -- node[left] {$\rho$} (normalization.south);
\end{tikzpicture}
\caption{Closed-loop adaptive prescribed performance architecture under input saturation. The mismatch between the requested and applied inputs adapts the performance functions, which recover their nominal values when saturation is inactive.}
\label{fig:closed_loop_architecture}
\end{figure*}
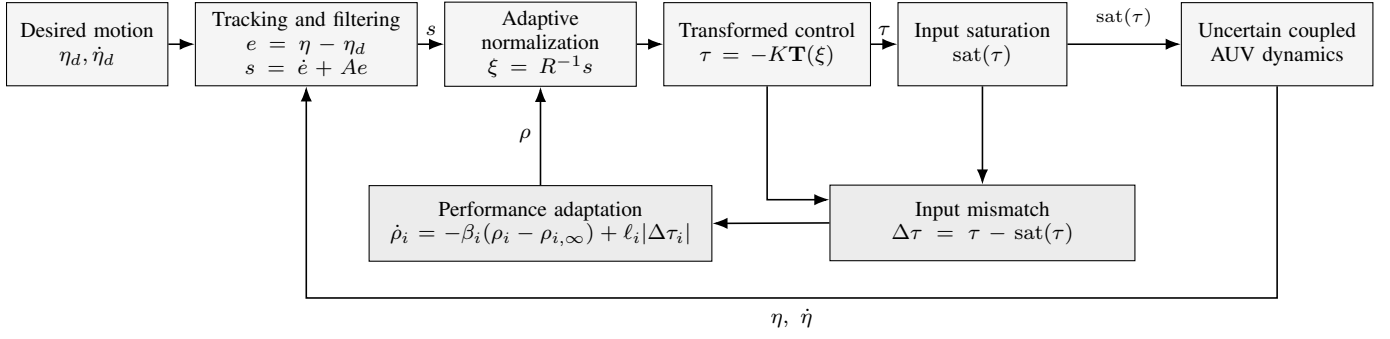

\subsection{Stability and Performance Analysis}
\label{subsection:Stability and Performance Analysis}

We next establish sufficient feasibility conditions under which the proposed controller guarantees bounded closed-loop operation, maintains the filtered errors within their adaptive performance bounds, and quantifies the degradation induced by input saturation.

\begin{theorem}
\label{Theorem_1}
Consider the Euler--Lagrange system \eqref{eq:agent dynamics} under Assumptions~\ref{Assumption: regularity}--\ref{Assumption: Reference Bounded}. Let the adaptive performance functions \eqref{eq:adaptive performance function} and the control law \eqref{eq:control law vector form} be applied.

\textit{Auxiliary quantities:} For any constants $\Lambda_i>0$ and $\varepsilon_i\in(0,1)$, $i=1,\ldots,m$, define $G_s(\eta,\dot\eta):=[M^{-1}(\eta,\dot\eta)+M^{-T}(\eta,\dot\eta)]/2$ and
\begin{align}
H(t,\eta,\dot\eta)
:=&
-M^{-1}(\eta,\dot\eta)
\left(
C(\eta,\dot\eta)\dot\eta
+
N(t,\eta,\dot\eta)
\right)
\nonumber\\
&-
\ddot\eta_d(t)
+
A\left(\dot\eta-\dot\eta_d(t)\right).
\label{eq:H theorem}
\end{align}
Further, let $\bar\rho_i^\Lambda:=\max\{\rho_i(0),\rho_{i,\infty}+\ell_i\Lambda_i/\beta_i\}$, $\bar s_i^\Lambda:=\bar\rho_i^\Lambda$, $\bar e_i^\Lambda:=|e_i(0)|+\bar s_i^\Lambda/\alpha_i$, $\bar v_i^\Lambda:=\bar s_i^\Lambda+\alpha_i\bar e_i^\Lambda$, $a_i^\Lambda:=\beta_i(\bar\rho_i^\Lambda-\rho_{i,\infty})$, $a^\Lambda:=[a_1^\Lambda,\ldots,a_m^\Lambda]^T$, and
\begin{equation}
\begin{aligned}
S
:=
\Big\{
(\eta,\dot\eta):
&
|\eta_i|
\le
\bar\eta_{d,i}+\bar e_i^\Lambda,
|\dot\eta_i|
\le
\bar{\dot\eta}_{d,i}+\bar v_i^\Lambda,
\\
&
i=1,\ldots,m
\Big\},
\end{aligned}
\label{eq:S theorem}
\end{equation}
where $\bar\eta_{d,i}$ and $\bar{\dot\eta}_{d,i}$ denote finite bounds of $\eta_{d,i}$ and $\dot\eta_{d,i}$, respectively. The constants $\Lambda_i$ serve as auxiliary mismatch margins in the feasibility test.

Let $\bar\rho^\Lambda:=[\bar\rho_1^\Lambda,\ldots,\bar\rho_m^\Lambda]^T$ and define
\begin{align}
D^\Lambda
:=&
\sup_{t\ge0}
\max_{(\eta,\dot\eta)\in S}
\frac{
\|H(t,\eta,\dot\eta)\|+\|a^\Lambda\|
}{
\lambda_{\min}\big(G_s(\eta,\dot\eta)\big)
},
\label{eq:D Lambda theorem}
\\
\widehat V_0
:=&
\sum_{i=1}^m
\rho_i(0)k_i|\xi_i(0)|T\big(|\xi_i(0)|\big),
\label{eq:V0 upper theorem}
\\
C^\Lambda
:=&
\max\left\{
\widehat V_0,
\|\bar\rho^\Lambda\|D^\Lambda
\right\}.
\label{eq:C Lambda theorem}
\end{align}

\textit{Feasibility condition:} If $\varepsilon_i$ are selected such that
\begin{align}
k_iT\left(\frac{\varepsilon_i}{2}\right)
>
\frac{2C^\Lambda}{\rho_{i,\infty}\varepsilon_i},
\qquad i=1,\ldots,m,
\label{eq:epsilon feasibility condition}
\end{align}
and the input saturation limits satisfy
\begin{align}
\tau_{i,\max}
>
\max\left\{
k_iT(\varepsilon_i)-\Lambda_i,
\frac{2C^\Lambda}{\rho_{i,\infty}\varepsilon_i}
\right\},
\qquad i=1,\ldots,m,
\label{eq:feasibility condition}
\end{align}
then the following conclusions hold.

\textit{(i) Boundedness and prescribed performance:} All closed-loop signals remain bounded for all $t\ge0$ and
\begin{equation}
|\xi_i(t)|<1,
\qquad
\forall t\ge0,
\quad i=1,\ldots,m.
\label{eq:xi theorem result}
\end{equation}
More specifically, the transformed errors satisfy
\begin{equation}
|T(\xi_i(t))|
<
\frac{\tau_{i,\max}+\Lambda_i}{k_i},
\qquad
\forall t\ge0,
\quad i=1,\ldots,m.
\label{eq:T xi theorem result}
\end{equation}
Consequently,
\begin{equation}
|s_i(t)|<\rho_i(t),
\qquad
\forall t\ge0,
\quad i=1,\ldots,m.
\label{eq:s theorem result}
\end{equation}

\textit{(ii) Adaptive performance bounds:} The adaptive performance functions satisfy
\begin{align}
\rho_i(t)
\le&
\rho_{i,\infty}
+
\left(\rho_i(0)-\rho_{i,\infty}\right)e^{-\beta_i t}
\nonumber\\
&+
\frac{\ell_i}{\beta_i}
\sup_{r\ge0}|\Delta\tau_i(r)|
\left(1-e^{-\beta_i t}\right).
\label{eq:rho theorem result}
\end{align}

\textit{(iii) Position tracking bounds:} The position tracking errors satisfy
\begin{equation}
|e_i(t)|
\le
|e_i(0)|e^{-\alpha_i t}
+
\int_0^t e^{-\alpha_i(t-r)}\rho_i(r)dr.
\label{eq:e theorem result}
\end{equation}
Thus,
\begin{equation}
\limsup_{t\to\infty}|e_i(t)|
\le
\frac{1}{\alpha_i}
\left(
\rho_{i,\infty}
+
\frac{\ell_i}{\beta_i}
\sup_{t\ge0}|\Delta\tau_i(t)|
\right).
\label{eq:e limsup theorem result}
\end{equation}
\end{theorem}

\begin{proof}
By Assumption~\ref{Assumption: M symmetric part positive definite}, $M(\eta,\dot\eta)$ is nonsingular and the dynamics can be written as
\begin{equation}
\ddot\eta
=
f(t,\eta,\dot\eta)
+
G(\eta,\dot\eta)\operatorname{sat}(\tau),
\label{eq:state space proof}
\end{equation}
where
\begin{align}
f(t,\eta,\dot\eta)
&:=
-M^{-1}(\eta,\dot\eta)
\left(
C(\eta,\dot\eta)\dot\eta
+
N(t,\eta,\dot\eta)
\right),
\label{eq:f definition proof}
\\
G(\eta,\dot\eta)
&:=
M^{-1}(\eta,\dot\eta).
\label{eq:G definition proof}
\end{align}
Moreover, Assumption~\ref{Assumption: M symmetric part positive definite} implies that $G_s=(G+G^T)/2$ is positive definite. Indeed, for every $y\in\mathbb R^m\setminus\{0\}$, letting $x=M^{-1}y$ yields $y^TG_sy=y^TGy=x^TM^Tx=x^TM_sx>0$.
From \eqref{eq:filtered error vector},
\begin{equation}
\dot e
=
s-Ae
=
R(t)\xi-Ae.
\label{eq:e dynamics proof}
\end{equation}
Since the control law is $\tau=-K\mathbf T(\xi)$ and the saturation is odd and componentwise, $\operatorname{sat}(\tau)=-\operatorname{sat}(K\mathbf T(\xi))$.
Differentiating \eqref{eq:filtered error vector} and using \eqref{eq:state space proof}, we obtain $\dot s=H(t,\eta,\dot\eta)-G(\eta,\dot\eta)\operatorname{sat}(K\mathbf T(\xi))$, where $H$ is given by \eqref{eq:H theorem}.
Differentiating $\xi=R^{-1}s$ gives
\begin{align}
\dot\xi
=
R^{-1}
\Big[
H(t,\eta,\dot\eta)
-
G(\eta,\dot\eta)
\operatorname{sat}\left(K\mathbf T(\xi)\right)
-
\dot R\xi
\Big].
\label{eq:xi dynamics proof}
\end{align}
Finally, the adaptive performance dynamics are
\begin{equation}
\dot\rho_i
=
-\beta_i(\rho_i-\rho_{i,\infty})
+
\ell_i|\Delta\tau_i|,
\quad i=1,\ldots,m.
\label{eq:rho dynamics proof}
\end{equation}

Let $\chi:=[e^T,\xi^T,\rho^T]^T$, where $\rho:=[\rho_1,\ldots,\rho_m]^T$.
The inclusion of $e$ in $\chi$ is necessary, since $\xi$ and $\rho$ determine $s=R\xi$, whereas $\dot e=s-Ae$ additionally requires $e$.
Consider the open set $\Omega:=\mathbb R^m\times(-1,1)^m\times(0,\infty)^m$. For every $\chi\in\Omega$, define $\eta_\chi(t):=e+\eta_d(t)$, $\dot\eta_\chi(t):=R\xi-Ae+\dot\eta_d(t)$, $\tau_\chi:=-K\mathbf T(\xi)$, and $\Delta\tau_\chi:=\tau_\chi-\operatorname{sat}(\tau_\chi)$, where $R=\operatorname{diag}(\rho_1,\ldots,\rho_m)$. Let also $\gamma_i(t,\chi):=-\beta_i(\rho_i-\rho_{i,\infty})+\ell_i|\Delta\tau_{\chi,i}|$, $i=1,\ldots,m$, with $\gamma(t,\chi):=[\gamma_1(t,\chi),\ldots,\gamma_m(t,\chi)]^T$, and define $\Gamma(t,\chi):=\operatorname{diag}(\gamma_1(t,\chi),\ldots,\gamma_m(t,\chi))$.
Then \eqref{eq:e dynamics proof}, \eqref{eq:xi dynamics proof}, and \eqref{eq:rho dynamics proof} can be written compactly as
\begin{equation}
\dot\chi
=
\mathcal F(t,\chi),
\label{eq:closed loop chi dynamics proof}
\end{equation}
where
\begin{equation}
\mathcal F(t,\chi)
:=
\begin{bmatrix}
R\xi-Ae
\\
\begin{aligned}
R^{-1}\big[&H(t,\eta_\chi,\dot\eta_\chi)\\[-1mm]
&-G(\eta_\chi,\dot\eta_\chi)
\operatorname{sat}\left(K\mathbf T(\xi)\right)
-\Gamma(t,\chi)\xi\big]
\end{aligned}
\\
\gamma(t,\chi)
\end{bmatrix}.
\label{eq:F chi proof}
\end{equation}
In view of Assumptions~\ref{Assumption: regularity}--\ref{Assumption: Reference Bounded}, the vector field $\mathcal F$ in \eqref{eq:closed loop chi dynamics proof} is continuous with respect to time and locally Lipschitz with respect to $\chi$ on $\Omega$. Moreover, by the selection \eqref{eq:rho initial condition}, it holds that $|\xi_i(0)|<1$, $i=1,\ldots,m$, and $\rho_i(0)>\rho_{i,\infty}>0$. Hence, $\chi(0)\in\Omega$. Thus, \eqref{eq:closed loop chi dynamics proof} admits a unique maximal solution $\chi:[0,t_{\max})\to\Omega$, where $t_{\max}\in(0,\infty]$.

The solution of \eqref{eq:rho dynamics proof} and \eqref{eq:rho initial condition} satisfies $\rho_i(t)\ge\rho_{i,\infty}$ on $[0,t_{\max})$. Define, for each $i=1,\ldots,m$, $z_i(r):=\operatorname{sat}(k_iT(r))$, $r\in(-1,1)$, and $F_i(r):=\int_0^r z_i(\sigma)d\sigma$. Since $T$ is strictly increasing and $T(0)=0$, and since the saturation function is odd and monotone, $z_i(r)$ has the same sign as $r$. Hence,
\begin{equation}
0\le F_i(r)\le z_i(r)r,
\qquad r\in(-1,1).
\label{eq:F inequality proof}
\end{equation}
Consider the auxiliary Lyapunov function
\begin{equation}
V(t,\xi)
:=
\sum_{i=1}^m
\rho_i(t)F_i(\xi_i).
\label{eq:V proof}
\end{equation}
By \eqref{eq:F inequality proof}, $V(0)\le\widehat V_0\le C^\Lambda$. Let
\begin{equation}
\theta_i^\Lambda
:=
T^{-1}\left(\frac{\tau_{i,\max}+\Lambda_i}{k_i}\right),
\qquad i=1,\ldots,m.
\label{eq:theta Lambda proof}
\end{equation}
The first term of \eqref{eq:feasibility condition} gives $\theta_i^\Lambda>\varepsilon_i$. Moreover, for every $r\in[\varepsilon_i/2,\varepsilon_i]$, \eqref{eq:epsilon feasibility condition} and the second term of \eqref{eq:feasibility condition} give
\begin{equation}
z_i(r)
>
\frac{2C^\Lambda}{\rho_{i,\infty}\varepsilon_i}.
\label{eq:z epsilon lower proof}
\end{equation}
Consequently,
\begin{align}
\rho_{i,\infty}F_i(\theta_i^\Lambda)
&\ge
\rho_{i,\infty}
\int_{\varepsilon_i/2}^{\varepsilon_i}z_i(r)dr
>
C^\Lambda.
\label{eq:F theta lower proof}
\end{align}
Thus, there exists a constant $c_\Lambda$ such that
\begin{equation}
C^\Lambda
<
c_\Lambda
<
\min_{1\le i\le m}
\rho_{i,\infty}F_i(\theta_i^\Lambda).
\label{eq:c Lambda proof}
\end{equation}
Define the sublevel set
\begin{equation}
\mathcal L_\Lambda
:=
\left\{
(\rho,\xi)\in
\prod_{i=1}^m[\rho_{i,\infty},\infty)\times(-1,1)^m:
V(\rho,\xi)\le c_\Lambda
\right\}.
\label{eq:L Lambda proof}
\end{equation}
Since $V(0)<c_\Lambda$, the initial condition lies in its interior relative to $\rho_i\ge\rho_{i,\infty}$. Furthermore, \eqref{eq:c Lambda proof} implies that every $(\rho,\xi)\in\mathcal L_\Lambda$ satisfies $|\xi_i|<\theta_i^\Lambda$, and therefore $|\tau_i|<\tau_{i,\max}+\Lambda_i$.

There exists a nonempty interval $[0,t_s)\subseteq[0,t_{\max})$ such that $(\rho(t),\xi(t))\in\mathcal L_\Lambda$ for all $t\in[0,t_s)$. On this interval, $|\tau_i(t)|<\tau_{i,\max}+\Lambda_i$ for every $i=1,\ldots,m$.
Thus,
\begin{equation}
|\Delta\tau_i(t)|
\le
\Lambda_i,
\qquad
\forall t\in[0,t_s),
\quad i=1,\ldots,m.
\label{eq:Delta bound proof}
\end{equation}
Combining \eqref{eq:Delta bound proof} with \eqref{eq:rho dynamics proof}, we obtain $\dot\rho_i\le-\beta_i(\rho_i-\rho_{i,\infty})+\ell_i\Lambda_i$.
By scalar comparison,
\begin{align}
\rho_i(t)
\le&
\rho_{i,\infty}
+
\left(\rho_i(0)-\rho_{i,\infty}\right)e^{-\beta_i t}
\nonumber\\
&+
\frac{\ell_i\Lambda_i}{\beta_i}
\left(1-e^{-\beta_i t}\right)
\le
\bar\rho_i^\Lambda,
\label{eq:rho bound proof}
\end{align}
for all $t\in[0,t_s)$ and $i=1,\ldots,m$.

Moreover, since $|\xi_i(t)|<\theta_i^\Lambda<1$, we have $|s_i(t)|=\rho_i(t)|\xi_i(t)|\le\bar\rho_i^\Lambda=\bar s_i^\Lambda$. Since $s_i=\dot e_i+\alpha_i e_i$, we have $e_i(t)=e_i(0)e^{-\alpha_i t}+\int_0^t e^{-\alpha_i(t-r)}s_i(r)dr$. Thus, $|e_i(t)|\le|e_i(0)|+\bar s_i^\Lambda/\alpha_i=\bar e_i^\Lambda$ and $|\dot e_i(t)|\le|s_i(t)|+\alpha_i|e_i(t)|\le\bar s_i^\Lambda+\alpha_i\bar e_i^\Lambda=\bar v_i^\Lambda$. Consequently, by Assumption~\ref{Assumption: Reference Bounded}, $(\eta(t),\dot\eta(t))\in S$ for all $t\in[0,t_s)$.

We now prove that no trajectory can leave $\mathcal L_\Lambda$.
For compactness of notation in the rest of the proof, let $z(\xi):=\operatorname{sat}(K\mathbf T(\xi))=[z_1(\xi_1),\ldots,z_m(\xi_m)]^T$.
Differentiating \eqref{eq:V proof} and using \eqref{eq:xi dynamics proof}, we obtain
\begin{align}
\dot V
=&
z^T
\left(
H-Gz-\dot R\xi
\right)
+
\sum_{i=1}^m\dot\rho_iF_i(\xi_i)
\nonumber\\
=&
z^TH
-
z^TGz
+
\sum_{i=1}^m
\dot\rho_i
\left[
F_i(\xi_i)-z_i(\xi_i)\xi_i
\right].
\label{eq:Vdot exact proof}
\end{align}
Substituting \eqref{eq:rho dynamics proof} into the last term yields
\begin{align}
&\sum_{i=1}^m
\dot\rho_i
\left[
F_i(\xi_i)-z_i(\xi_i)\xi_i
\right]
\nonumber\\
=&
\sum_{i=1}^m
\left[
-\beta_i(\rho_i-\rho_{i,\infty})
+
\ell_i|\Delta\tau_i|
\right]
\left[
F_i(\xi_i)-z_i(\xi_i)\xi_i
\right]
\nonumber\\
=&
\sum_{i=1}^m
\beta_i(\rho_i-\rho_{i,\infty})
\left[
z_i(\xi_i)\xi_i-F_i(\xi_i)
\right]
\nonumber\\
&-
\sum_{i=1}^m
\ell_i|\Delta\tau_i|
\left[
z_i(\xi_i)\xi_i-F_i(\xi_i)
\right].
\label{eq:rho terms proof}
\end{align}
By \eqref{eq:F inequality proof}, the last term in \eqref{eq:rho terms proof} is nonpositive. Therefore,
\begin{align}
\dot V
\le&
z^TH
-
z^TGz
\nonumber\\
&+
\sum_{i=1}^m
\beta_i(\rho_i-\rho_{i,\infty})
\left[
z_i(\xi_i)\xi_i-F_i(\xi_i)
\right].
\label{eq:Vdot after rho proof}
\end{align}
Furthermore, on $[0,t_s)$, \eqref{eq:rho dynamics proof} and \eqref{eq:rho initial condition} imply $\rho_i(t)\ge\rho_{i,\infty}$, whereas \eqref{eq:rho bound proof} gives $\beta_i(\rho_i-\rho_{i,\infty})\le a_i^\Lambda$. Since $F_i(\xi_i)\ge0$, $\beta_i(\rho_i-\rho_{i,\infty})\ge0$, and $z_i(\xi_i)\xi_i=|z_i(\xi_i)||\xi_i|\le |z_i(\xi_i)|$, we obtain
\begin{align}
&\sum_{i=1}^m
\beta_i(\rho_i-\rho_{i,\infty})
\left[
z_i(\xi_i)\xi_i-F_i(\xi_i)
\right]
\nonumber\\
&\le
\sum_{i=1}^m
\beta_i(\rho_i-\rho_{i,\infty})
z_i(\xi_i)\xi_i
\nonumber\\
&\le
\sum_{i=1}^m
a_i^\Lambda |z_i(\xi_i)|
\le
\|a^\Lambda\|\,\|z(\xi)\|.
\label{eq:a bound proof}
\end{align}
Using the Cauchy--Schwarz inequality, the identity $z^TGz=z^TG_sz$, and the positive definiteness of $G_s$ on $S$, we obtain
\begin{align}
\dot V
\le&
-\lambda_{\min}\big(G_s(\eta,\dot\eta)\big)\|z\|^2
\nonumber\\
&+
\|z\|
\left(
\|H(t,\eta,\dot\eta)\|+\|a^\Lambda\|
\right).
\label{eq:Vdot bound proof}
\end{align}

The faces $\rho_i=\rho_{i,\infty}$ cannot be crossed outward because \eqref{eq:rho dynamics proof} gives $\dot\rho_i=\ell_i|\Delta\tau_i|\ge0$ there. Hence, at a first exit point from $\mathcal L_\Lambda$, it holds that $V=c_\Lambda$; moreover, \eqref{eq:rho bound proof} holds there by continuity. By \eqref{eq:F inequality proof} and the Cauchy--Schwarz inequality,
\begin{align}
c_\Lambda
&=
\sum_{i=1}^m\rho_iF_i(\xi_i)
\le
\sum_{i=1}^m\rho_i|z_i(\xi_i)|
\nonumber\\
&\le
\|\rho\|\,\|z(\xi)\|
\le
\|\bar\rho^\Lambda\|\,\|z(\xi)\|.
\label{eq:z lower sublevel proof}
\end{align}
It follows from \eqref{eq:c Lambda proof} and \eqref{eq:C Lambda theorem} that
\begin{equation}
\|z(\xi)\|
\ge
\frac{c_\Lambda}{\|\bar\rho^\Lambda\|}
>
D^\Lambda.
\label{eq:z D lower proof}
\end{equation}
Invoking \eqref{eq:D Lambda theorem} in \eqref{eq:Vdot bound proof} therefore gives $\dot V<0$ whenever a trajectory reaches $V=c_\Lambda$ from the interior of $\mathcal L_\Lambda$. This contradicts a first outward crossing of the sublevel boundary. Hence, $\mathcal L_\Lambda$ is positively invariant, $t_s=t_{\max}$, and
\begin{equation}
(\rho(t),\xi(t))\in\mathcal L_\Lambda,
\qquad
\forall t\in[0,t_{\max}).
\label{eq:L invariant proof}
\end{equation}
Consequently, $|\xi_i(t)|<\theta_i^\Lambda<1$ for all $t\in[0,t_{\max})$ and $i=1,\ldots,m$. It follows from \eqref{eq:xi scalar definition} that $|s_i(t)|<\rho_i(t)$ for all $t\in[0,t_{\max})$ and $i=1,\ldots,m$.
Moreover, \eqref{eq:L invariant proof} and \eqref{eq:theta Lambda proof} imply boundedness of $\mathbf T(\xi)$, and hence boundedness of $\tau$, $\operatorname{sat}(\tau)$, and $\Delta\tau$. The bounds \eqref{eq:rho bound proof} then hold on $[0,t_{\max})$, so $\rho_i(t)$ is bounded for every $i=1,\ldots,m$. Hence, $s_i=\rho_i\xi_i$ is bounded. Since $s_i=\dot e_i+\alpha_i e_i$, both $e_i$ and $\dot e_i$ are bounded. Therefore, $\eta$ and $\dot\eta$ are bounded. Thus, all closed loop signals remain bounded on $[0,t_{\max})$.

Since $|\xi_i(t)|<\theta_i^\Lambda<1$ and the solution of \eqref{eq:rho dynamics proof} satisfies
\begin{align*}
\rho_i(t)
=&\ \rho_{i,\infty}
+(\rho_i(0)-\rho_{i,\infty})e^{-\beta_i t}
\\
&+\ell_i\int_0^t e^{-\beta_i(t-r)}|\Delta\tau_i(r)|dr
\ge \rho_{i,\infty}>0,
\end{align*}
the solution remains in a compact subset of $\Omega$. Hence, it cannot escape the domain of definition in finite time. Consequently, $t_{\max}=\infty$.
This proves \eqref{eq:xi theorem result}--\eqref{eq:s theorem result}.

Finally, solving \eqref{eq:rho dynamics proof} gives
\begin{align}
\rho_i(t)
=
&
\rho_{i,\infty}
+
\left(\rho_i(0)-\rho_{i,\infty}\right)e^{-\beta_i t}
\nonumber\\
&+
\ell_i
\int_0^t
e^{-\beta_i(t-r)}
|\Delta\tau_i(r)|dr.
\label{eq:rho solution proof}
\end{align}
Thus,
\begin{align}
\rho_i(t)
\le&
\rho_{i,\infty}
+
\left(\rho_i(0)-\rho_{i,\infty}\right)e^{-\beta_i t}
\nonumber\\
&+
\frac{\ell_i}{\beta_i}
\sup_{r\ge0}|\Delta\tau_i(r)|
\left(1-e^{-\beta_i t}\right),
\label{eq:rho final proof}
\end{align}
which proves \eqref{eq:rho theorem result}.

Since $s_i=\dot e_i+\alpha_i e_i$, we have $e_i(t)=e_i(0)e^{-\alpha_i t}+\int_0^t e^{-\alpha_i(t-r)}s_i(r)dr$.
Using $|s_i(t)|<\rho_i(t)$ gives \eqref{eq:e theorem result}. Taking the limit superior in \eqref{eq:e theorem result} and using \eqref{eq:rho theorem result} gives \eqref{eq:e limsup theorem result}. The proof is complete.
\end{proof}

\begin{remark}
The theorem suggests a practical tuning procedure. First, $\rho_i(0)$ must be selected according to \eqref{eq:rho initial condition} to ensure that the normalized errors are initially well posed. One may then request fast recovery and tight nominal accuracy by choosing a large $\beta_i$ and a small $\rho_{i,\infty}$, respectively. If these specifications exceed the available actuation, the saturation-mismatch mechanism automatically relaxes the corresponding performance function toward achievable values while preserving the closed-loop guarantees.
\end{remark}

\section{Simulation and Experimental Validation}\label{section:SIMULATION}

This section evaluates the proposed controller in two stages. First, a comparative BlueROV2 simulation contrasts APPC with conventional PPC under identical trajectories and input limits; subsequently, a real-time pool experiment demonstrates the practical tracking performance of the adaptive design.

\subsection{Comparative Simulation}
The proposed adaptive PPC (APPC) and the corresponding PPC formulation implemented without envelope adaptation were compared using the same BlueROV2 model, reference, and input limits. The simulation setup consists of ROS~2 Humble, Gazebo Sim, ArduSub software in the loop, MAVROS, and \texttt{ros2\_control} with the six-thruster allocation matrix. The vehicle was spawned at $(x,y,z)=(0,0,-1)$~m. The controlled coordinates were $[x,y,-z,\phi,\psi]^T$, while pitch remained unactuated. Body-frame velocities were transformed through the vehicle kinematic Jacobian, and the filtered errors were mapped through its reduced inverse to the actuator coordinates $[X,Y,Z,K,N]^T$. The controller ran with $|\tau_i|\leq4$ in all five normalized actuator coordinates.

The desired reference is defined by $x_d=1.9+1.6\sin(\omega t)$, $y_d=-0.8+0.65\sin(2\omega t)$, $-z_d=2.5+0.10\sin(\omega t)$, $\phi_d=-20^{\circ}\sin(\omega t)$, and $\psi_d=\operatorname{atan2}(\dot y_d,\dot x_d)$, for $t\in[0,41.89)$~s, with $\omega=0.30~\mathrm{rad/s}$. For the controlled-coordinate ordering $[x,y,-z,\phi,\psi]$, $(\alpha_1,\ldots,\alpha_5)=(1,1,1,1,1)$. For the actuator-coordinate ordering $[X,Y,Z,K,N]$, the common parameters were $K=\operatorname{diag}(15,15,20,4,6)$, $(\beta_1,\ldots,\beta_5)=(0.1,0.1,0.2,0.1,0.2)$, $(\rho_1(0),\ldots,\rho_5(0))=(5.5,3.5,1,2,3.14)$, and $(\rho_{1,\infty},\ldots,\rho_{5,\infty})=(0.2,0.2,0.1,0.2,0.1)$. For APPC, $(\ell_1,\ldots,\ell_5)=(0.1,0.1,0.267,0.667,0.5)$, whereas for PPC the saturation-mismatch term was removed by setting $\ell_i=0$ for all $i$.

\begin{figure*}[ttpbh]
\centering
\begin{minipage}[t]{0.49\textwidth}
\centering
\includegraphics[width=\linewidth]{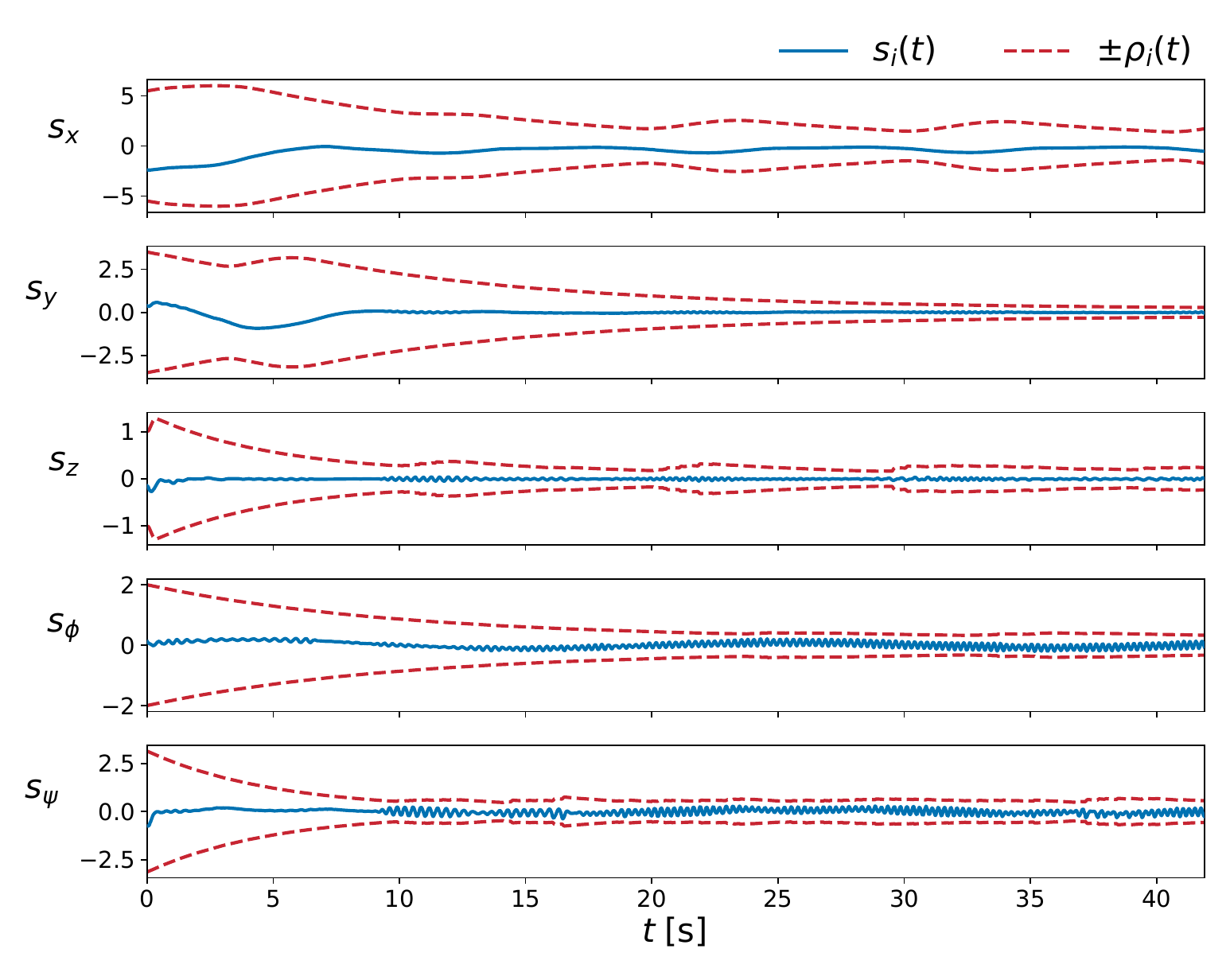}\\[-1mm]
\small (a)
\end{minipage}\hfill
\begin{minipage}[t]{0.49\textwidth}
\centering
\includegraphics[width=\linewidth]{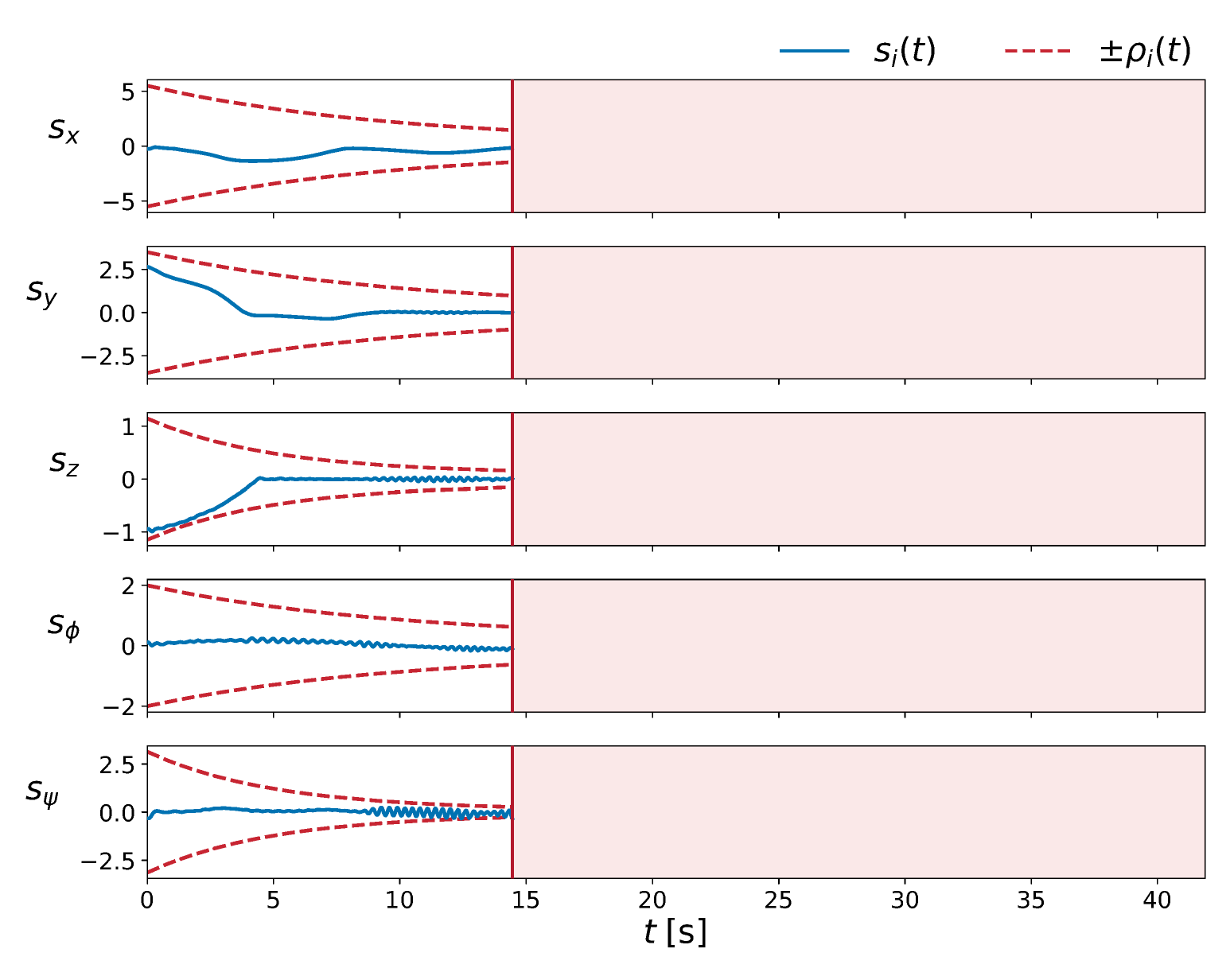}\\[-1mm]
\small (b)
\end{minipage}

\vspace{1mm}
\begin{minipage}[t]{0.49\textwidth}
\centering
\includegraphics[width=\linewidth]{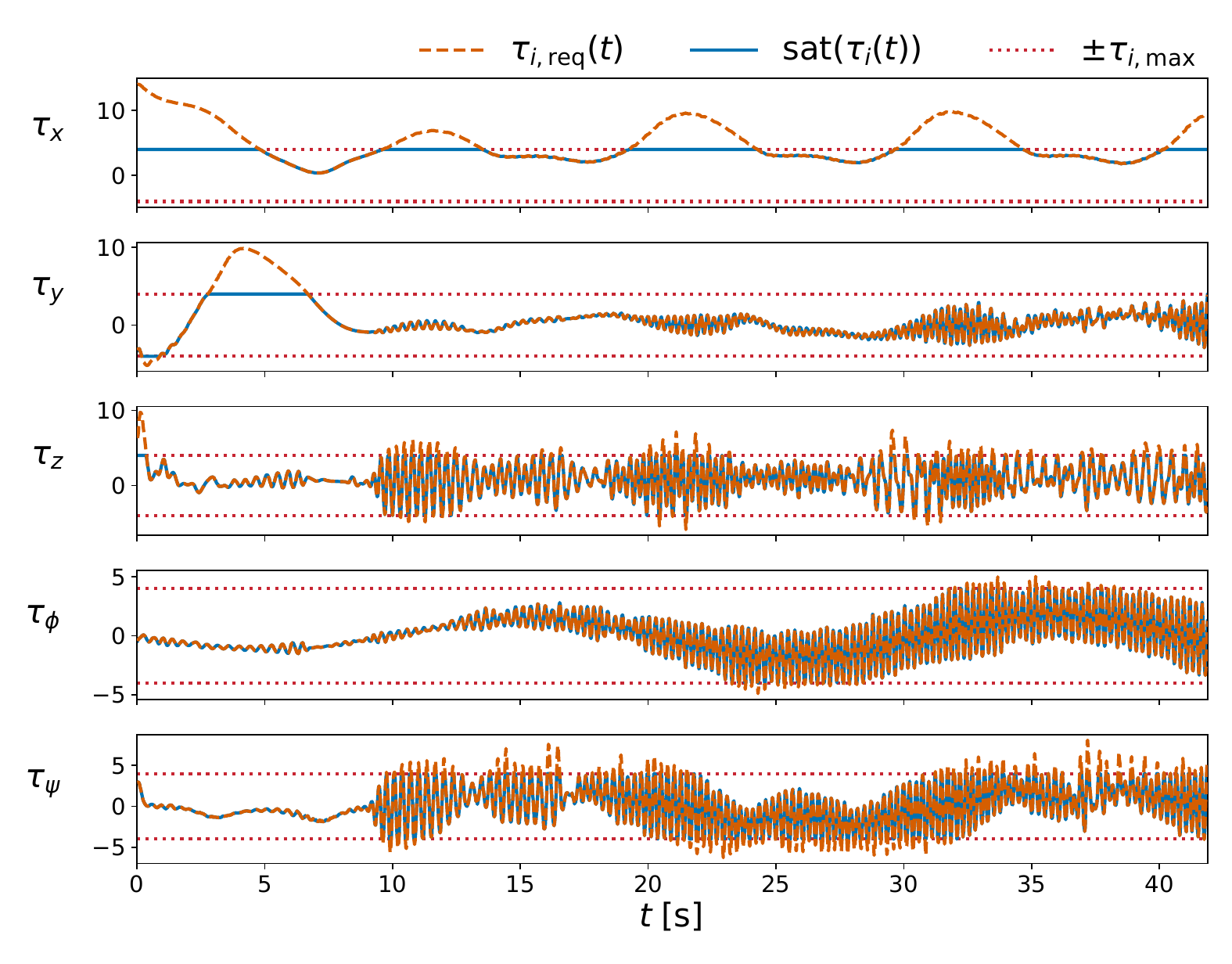}\\[-1mm]
\small (c)
\end{minipage}\hfill
\begin{minipage}[t]{0.49\textwidth}
\centering
\includegraphics[width=\linewidth]{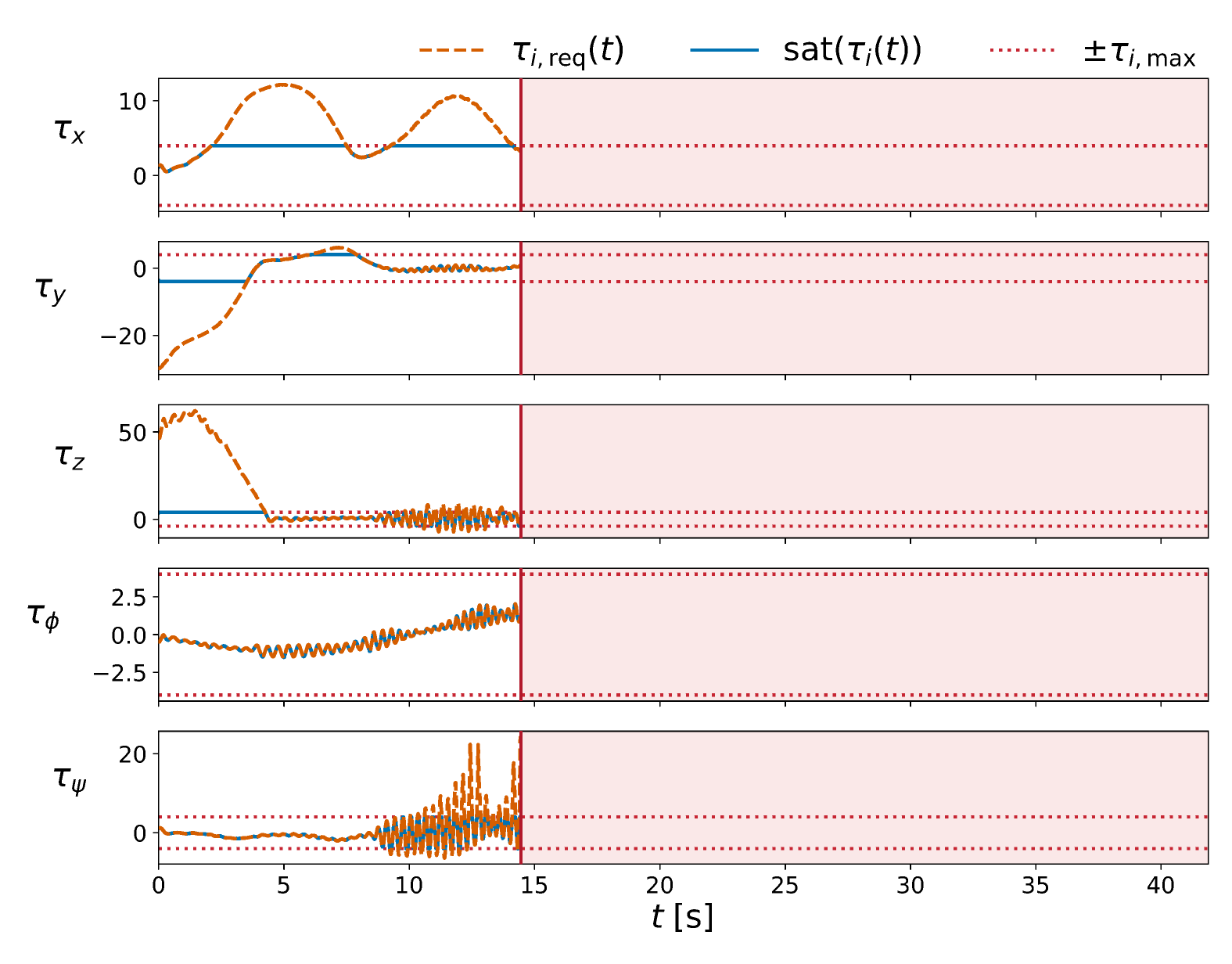}\\[-1mm]
\small (d)
\end{minipage}
\caption{Comparative BlueROV2 simulation under the common input limits $|\tau_i|\le4$. (a) APPC filtered errors and adaptive performance functions. (b) PPC filtered errors and fixed performance functions. (c) APPC requested and applied inputs. (d) PPC requested and applied inputs. APPC maintains prescribed performance, whereas PPC violates the yaw bound at $t=14.459$~s and the simulation is subsequently stopped.}
\label{fig:simulation_results}
\vspace{-2mm}
\end{figure*}

Adaptive PPC maintained $|s_i(t)|<\rho_i(t)$ and stable tracking throughout both periods despite sustained input saturation. Under the identical reference and input limits, the PPC without adaptation could not accommodate the loss of control authority, violated the yaw performance bound at $t=14.459$~s, and became unstable. Fig.~\ref{fig:simulation_results} compares the filtered errors, performance functions, and constrained inputs for both controllers; synchronized vehicle motion is shown in the ancillary file \texttt{SupplementaryVideo.mp4}.

\subsection{Experimental Validation}
Real-time validation was conducted with a tethered BlueROV2 in an indoor test tank. The controller ran under ROS and MAVROS, and the controlled variables were $[x,y,\psi]^T$. The planar Jacobian transformed the inertial-frame tracking errors to the body-frame filtered errors $[s_x^b,s_y^b,s_\psi]^T$, and the controller outputs were applied as PWM commands in surge, sway, and yaw. A synchronized external camera recorded the experiment.

State feedback was supplied by a ROS sensor-fusion pipeline. A calibrated vehicle-mounted camera observed a pre-mapped array of ArUco fiducials on the pool floor. The \texttt{aruco\_detect} package, based on the ArUco marker framework \cite{garrido2014automatic}, extracted the marker detections, while \texttt{fiducial\_slam} \cite{ubiquityrobotics_fiducials} combined them with the fixed map to estimate the absolute planar pose $(x,y,\psi)$. A two-dimensional EKF from the \texttt{robot\_localization} package \cite{moore2014generalized} fused this pose with planar optical-flow velocities and MAVROS IMU angular rates and linear accelerations. The resulting fused pose and twist were supplied to the controller. The EKF was solely an experimental measurement-fusion layer and was not part of the proposed control law.

The desired reference is defined by $x_d=1.9+1.6\sin(\omega t)$, $y_d=-0.8+0.65\sin(2\omega t)$, and $\psi_d=\operatorname{atan2}(\dot y_d,\dot x_d)$, with $\omega=0.03~\mathrm{rad/s}$. The filtered-error gains were $\alpha_x=\alpha_y=\alpha_\psi=1$. For $i=1,2,3$ corresponding to the surge, sway, and yaw control channels, respectively, the remaining parameters were $k_i=0.7$, $\beta_i=0.7$, $\rho_{i,\infty}=0.1$, $\ell_i=1$, and $\tau_{i,\max}=4$, with $(\rho_1(0),\rho_2(0),\rho_3(0))=(5.5,3.5,3.14)$.

\begin{figure}[tpbh]
\centering
\includegraphics[width=0.88\columnwidth]{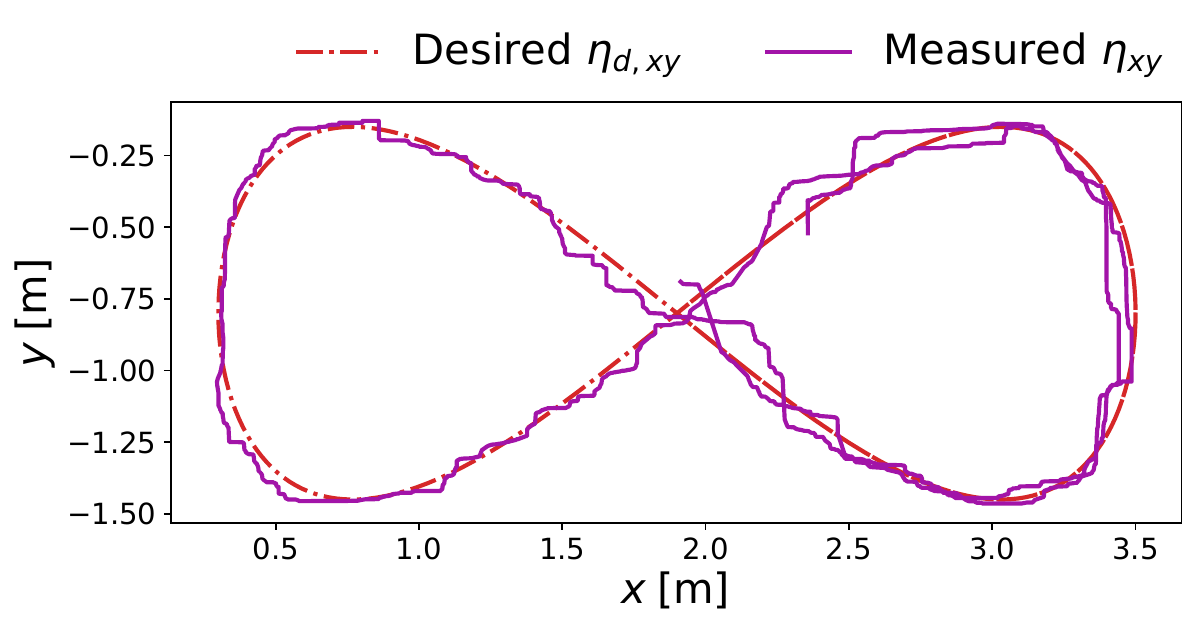}
\caption{Experimental planar trajectory tracking over the reported interval.}
\label{fig:experimental_trajectory}
\vspace{-2mm}
\end{figure}

Throughout the reported interval of $313.38$~s, the prescribed performance functions adapted to the saturation mismatch by enlarging the corresponding funnels whenever necessary. This relaxation preserved closed-loop stability and trajectory tracking in agreement with Theorem~\ref{Theorem_1}. Figs.~\ref{fig:bluerov2_platform}, \ref{fig:experimental_trajectory}, and \ref{fig:experimental_signals} present the experimental platform, planar trajectory, channelwise filtered errors, adaptive performance functions, and constrained inputs; synchronized vehicle motion through a complete measured figure-eight traversal is shown in the same supplementary video.

\begin{figure*}[ttpbh]
\centering
\begin{minipage}[t]{0.32\textwidth}
\centering
\includegraphics[width=0.90\linewidth]{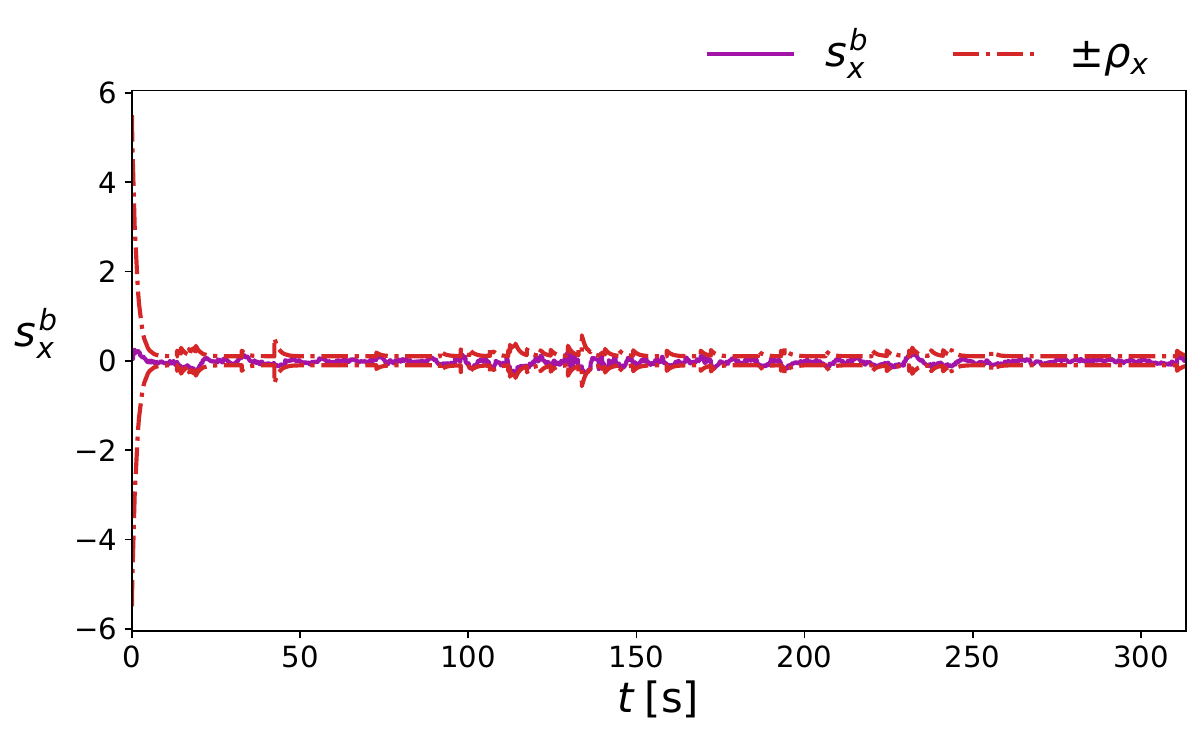}\\[-1mm]
\small (a)
\end{minipage}\hfill
\begin{minipage}[t]{0.32\textwidth}
\centering
\includegraphics[width=0.90\linewidth]{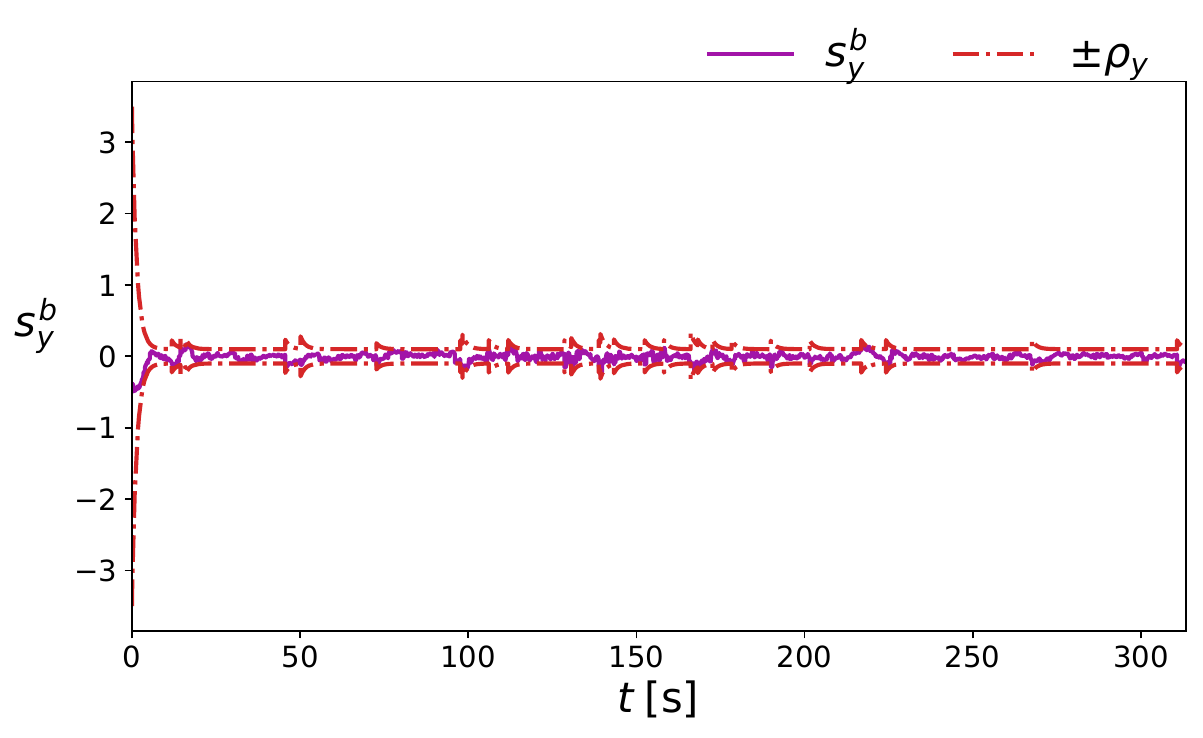}\\[-1mm]
\small (b)
\end{minipage}\hfill
\begin{minipage}[t]{0.32\textwidth}
\centering
\includegraphics[width=0.90\linewidth]{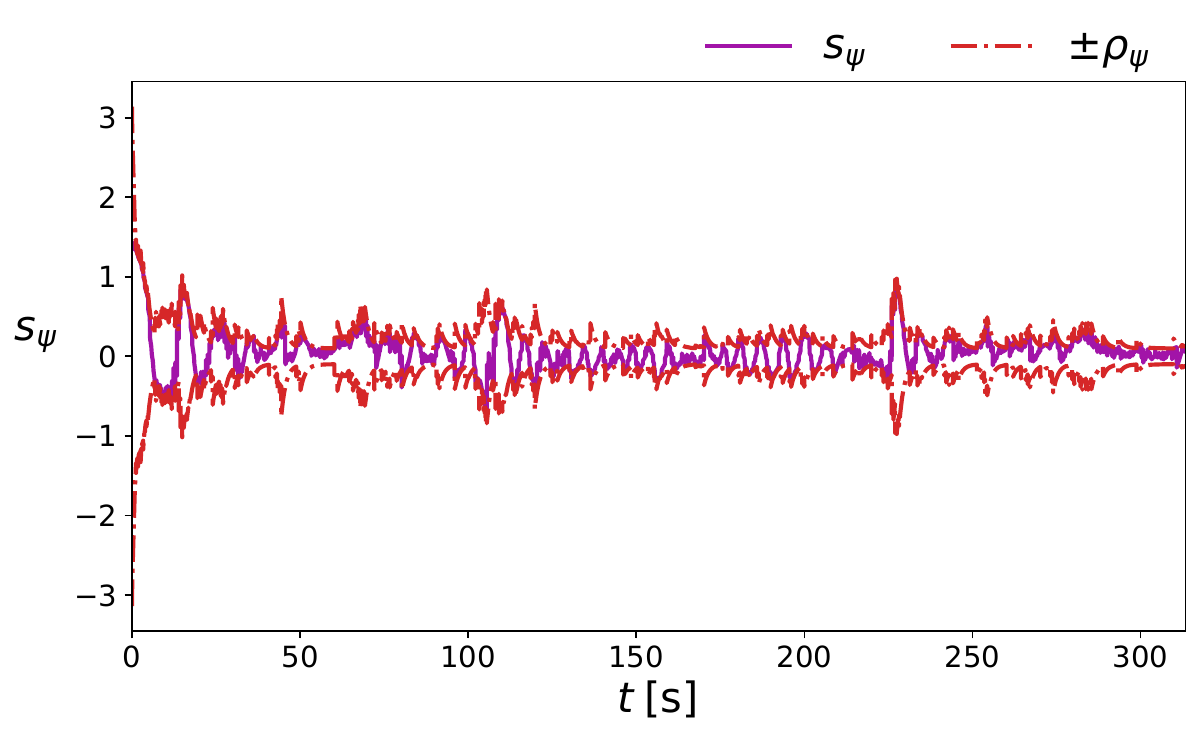}\\[-1mm]
\small (c)
\end{minipage}

\vspace{1mm}
\begin{minipage}[t]{0.32\textwidth}
\centering
\includegraphics[width=0.90\linewidth]{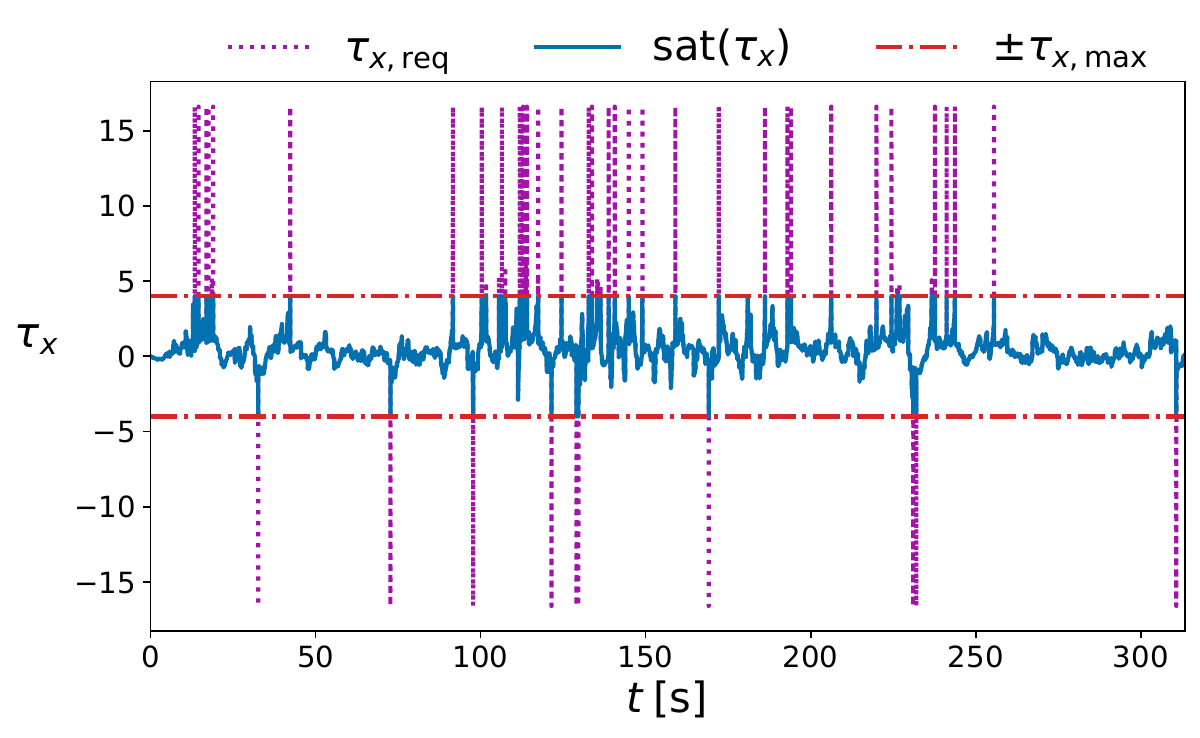}\\[-1mm]
\small (d)
\end{minipage}\hfill
\begin{minipage}[t]{0.32\textwidth}
\centering
\includegraphics[width=0.90\linewidth]{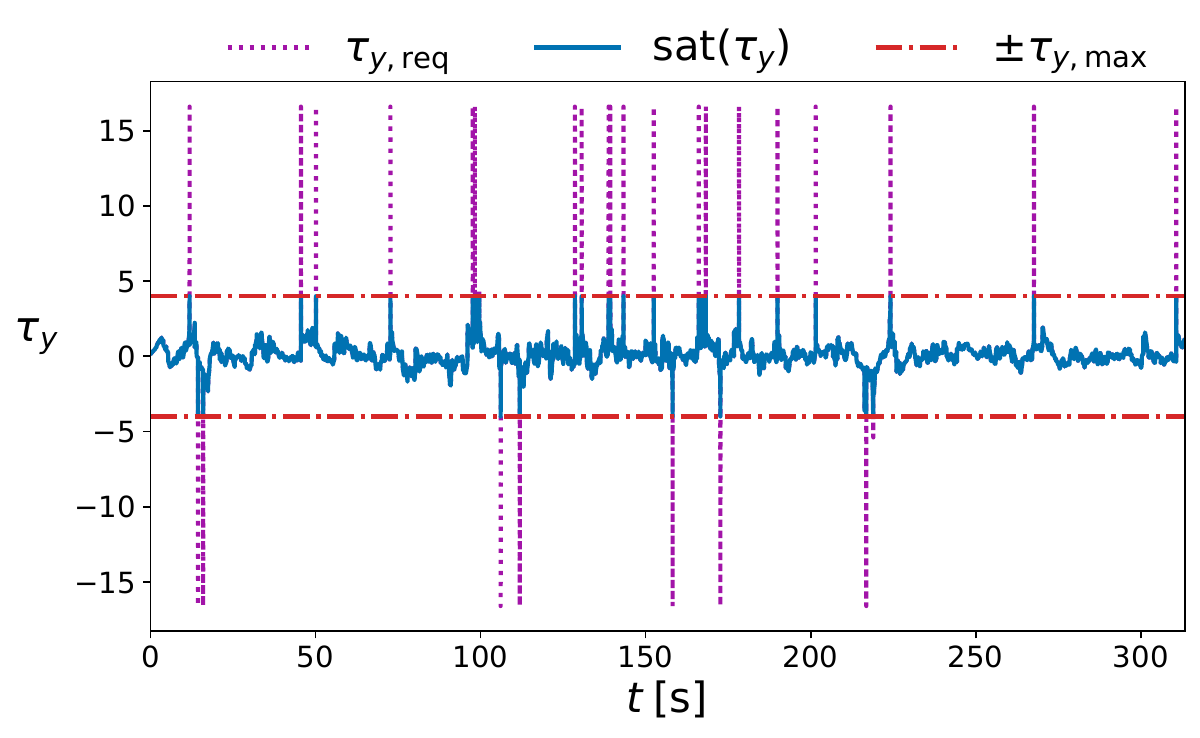}\\[-1mm]
\small (e)
\end{minipage}\hfill
\begin{minipage}[t]{0.32\textwidth}
\centering
\includegraphics[width=0.90\linewidth]{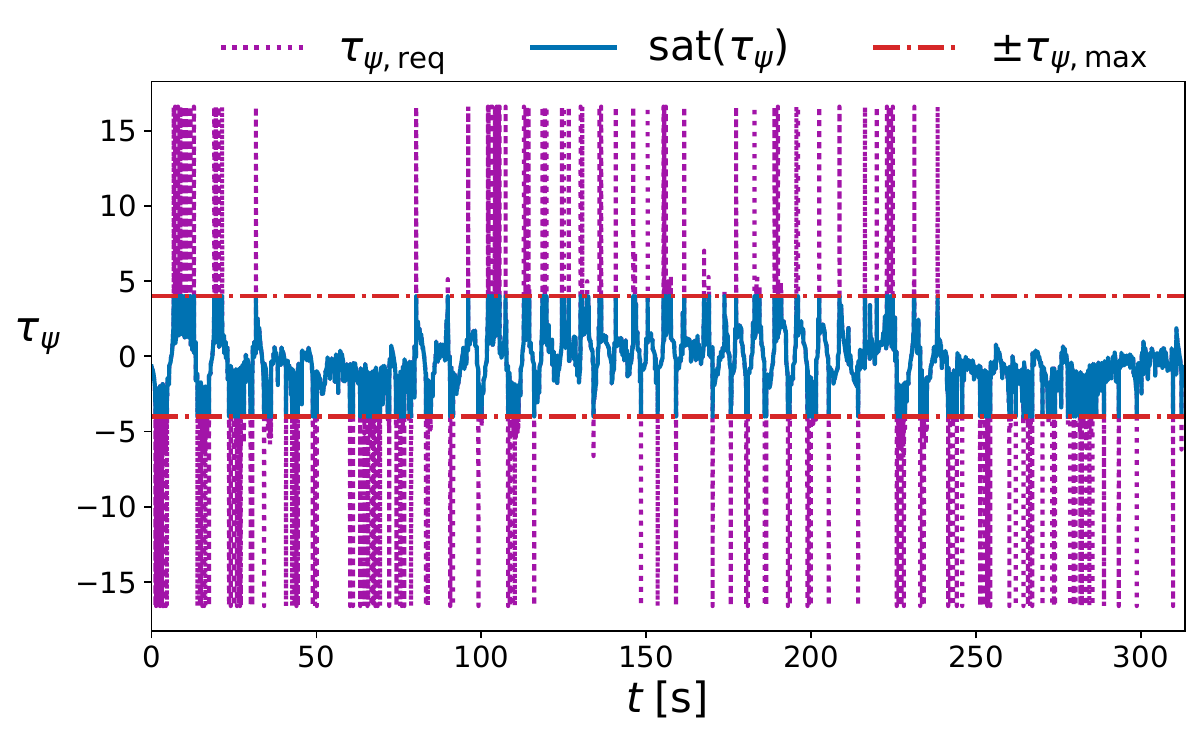}\\[-1mm]
\small (f)
\end{minipage}
\caption{Experimental closed-loop signals over the complete reported interval. (a)--(c) Filtered errors and adaptive performance functions for surge, sway, and yaw, respectively. (d)--(f) Corresponding requested and applied input commands under $|\tau_i|\le4$.}
\label{fig:experimental_signals}
\vspace{-2mm}
\end{figure*}

\section{Conclusion}\label{sec:conclusion}
This paper presented a low complexity APPC scheme for uncertain AUV dynamics with coupled saturated inputs. The analysis established boundedness and prescribed tracking performance, while the comparative simulation and BlueROV2 experiment demonstrated stable trajectory tracking under input saturation without modifying the reference.

\bibliographystyle{IEEEtran} 
\bibliography{IEEEabrv,references}

\begin{thebibliography}{10}
\providecommand{\url}[1]{#1}
\csname url@rmstyle\endcsname
\providecommand{\newblock}{\relax}
\providecommand{\bibinfo}[2]{#2}
\providecommand\BIBentrySTDinterwordspacing{\spaceskip=0pt\relax}
\providecommand\BIBentryALTinterwordstretchfactor{4}
\providecommand\BIBentryALTinterwordspacing{\spaceskip=\fontdimen2\font plus
\BIBentryALTinterwordstretchfactor\fontdimen3\font minus
  \fontdimen4\font\relax}
\providecommand\BIBforeignlanguage[2]{{%
\expandafter\ifx\csname l@#1\endcsname\relax
\typeout{** WARNING: IEEEtran.bst: No hyphenation pattern has been}%
\typeout{** loaded for the language `#1'. Using the pattern for}%
\typeout{** the default language instead.}%
\else
\language=\csname l@#1\endcsname
\fi
#2}}

\bibitem{BechlioulisKarras2017}
C.~P. Bechlioulis, G.~C. Karras, S.~Heshmati-Alamdari, and K.~J. Kyriakopoulos,
  ``Trajectory tracking with prescribed performance for underactuated
  underwater vehicles under model uncertainties and external disturbances,''
  \emph{IEEE Transactions on Control Systems Technology}, vol.~25, no.~2, pp.
  429--440, 2017.

\bibitem{lapandic2024kinodynamic}
D.~Lapandic, C.~K. Verginis, D.~V. Dimarogonas, and B.~Wahlberg, ``Kinodynamic
  motion planning via funnel control for underactuated unmanned surface
  vehicles,'' 2024, arXiv:2308.00130.

\bibitem{4639441}
C.~P. Bechlioulis and G.~A. Rovithakis, ``Robust adaptive control of feedback
  linearizable {MIMO} nonlinear systems with prescribed performance,''
  \emph{IEEE Transactions on Automatic Control}, vol.~53, no.~9, pp.
  2090--2099, 2008.

\bibitem{ilchmann2002tracking}
A.~Ilchmann, E.~P. Ryan, and C.~J. Sangwin, ``{Tracking with prescribed
  transient behaviour},'' \emph{ESAIM: Control, Optimisation and Calculus of
  Variations}, vol.~7, pp. 471--493, 2002.

\bibitem{karason1993adaptive}
S.~P. Karason and A.~M. Annaswamy, ``Adaptive control in the presence of input
  constraints,'' in \emph{1993 American Control Conference}.\hskip 1em plus
  0.5em minus 0.4em\relax IEEE, 1993, pp. 1370--1374.

\bibitem{CHEN2011452}
M.~Chen, S.~S. Ge, and B.~Ren, ``Adaptive tracking control of uncertain {MIMO}
  nonlinear systems with input constraints,'' \emph{Automatica}, vol.~47,
  no.~3, pp. 452--465, 2011.

\bibitem{chen2020robust}
C.~Chen, G.~Zhu, Q.~Zhang, and J.~Zhang, ``Robust adaptive finite-time tracking
  control for uncertain {Euler-Lagrange} systems with input saturation,''
  \emph{IEEE Access}, vol.~8, pp. 187\,605--187\,614, 2020.

\bibitem{wang2021neuroadaptive}
C.~Wang and T.~Kuang, ``Neuroadaptive control for uncertain {Euler-Lagrange}
  systems with input and output constraints,'' \emph{IEEE Access}, vol.~9, pp.
  51\,940--51\,949, 2021.

\bibitem{shao2021adaptive}
K.~Shao, R.~Tang, F.~Xu, X.~Wang, and J.~Zheng, ``Adaptive sliding mode control
  for uncertain {Euler--Lagrange} systems with input saturation,''
  \emph{Journal of the Franklin Institute}, vol. 358, no.~16, pp. 8356--8376,
  2021.

\bibitem{ilchmann2004input}
A.~Ilchmann and S.~Trenn, ``{Input constrained funnel control with applications
  to chemical reactor models},'' \emph{Systems \& Control Letters}, vol.~53,
  no.~5, pp. 361--375, 2004.

\bibitem{hu2022funnel}
J.~Hu, S.~Trenn, and X.~Zhu, ``Funnel control for relative degree one nonlinear
  systems with input saturation,'' in \emph{2022 European Control Conference
  (ECC)}.\hskip 1em plus 0.5em minus 0.4em\relax IEEE, 2022, pp. 227--232.

\bibitem{hopfe2010funnel}
N.~Hopfe, A.~Ilchmann, and E.~P. Ryan, ``{Funnel control with saturation:
  Linear MIMO systems},'' \emph{IEEE Transactions on Automatic Control},
  vol.~55, no.~2, pp. 532--538, 2010.

\bibitem{liberzon2010bang}
D.~Liberzon and S.~Trenn, ``{The bang-bang funnel controller},'' in \emph{49th
  IEEE Conference on Decision and Control (CDC)}.\hskip 1em plus 0.5em minus
  0.4em\relax IEEE, 2010, pp. 690--695.

\bibitem{liberzon2013bang}
------, ``{The bang-bang funnel controller for uncertain nonlinear systems with
  arbitrary relative degree},'' \emph{IEEE Transactions on Automatic Control},
  vol.~58, no.~12, pp. 3126--3141, 2013.

\bibitem{wang2019improved}
Y.~Wang, J.~Hu, J.~Li, and B.~Liu, ``Improved prescribed performance control
  for nonaffine pure-feedback systems with input saturation,''
  \emph{International Journal of Robust and Nonlinear Control}, vol.~29, no.~6,
  pp. 1769--1788, 2019.

\bibitem{ji2022saturation}
R.~Ji, D.~Li, J.~Ma, and S.~S. Ge, ``Saturation-tolerant prescribed control of
  {MIMO} systems with unknown control directions,'' \emph{IEEE Transactions on
  Fuzzy Systems}, vol.~30, no.~12, pp. 5116--5127, 2022.

\bibitem{ji2022saturation2}
R.~Ji, D.~Li, and S.~S. Ge, ``Saturation-tolerant prescribed control for
  nonlinear time-delay systems,'' \emph{IEEE Transactions on Fuzzy Systems},
  vol.~31, no.~8, pp. 2495--2508, 2022.

\bibitem{10004950}
P.~K. Mishra and P.~Jagtap, ``Approximation-free prescribed performance control
  with prescribed input constraints,'' \emph{IEEE Control Systems Letters},
  vol.~7, pp. 1261--1266, 2023.

\bibitem{YONG2020109268}
K.~Yong, M.~Chen, Y.~Shi, and Q.~Wu, ``Flexible performance-based robust
  control for a class of nonlinear systems with input saturation,''
  \emph{Automatica}, vol. 122, p. 109268, 2020.

\bibitem{trakas2023robust}
P.~S. Trakas and C.~P. Bechlioulis, ``Robust adaptive prescribed performance
  control for unknown nonlinear systems with input amplitude and rate
  constraints,'' \emph{IEEE Control Systems Letters}, vol.~7, pp. 1801--1806,
  2023.

\bibitem{10693631}
------, ``Adaptive performance control for input constrained {MIMO} nonlinear
  systems,'' \emph{IEEE Transactions on Systems, Man, and Cybernetics:
  Systems}, vol.~54, no.~12, pp. 7733--7745, 2024.

\bibitem{BIKAS2024111663}
L.~N. Bikas and G.~A. Rovithakis, ``Prescribed performance under input
  saturation for uncertain strict-feedback systems: A switching control
  approach,'' \emph{Automatica}, vol. 165, p. 111663, 2024.

\bibitem{li2018adaptive}
S.~Li and Z.~Xiang, ``Adaptive prescribed performance control for switched
  nonlinear systems with input saturation,'' \emph{International Journal of
  Systems Science}, vol.~49, no.~1, pp. 113--123, 2018.

\bibitem{cheng2019neural}
C.~Cheng, Y.~Zhang, and S.~Liu, ``Neural observer-based adaptive prescribed
  performance control for uncertain nonlinear systems with input saturation,''
  \emph{Neurocomputing}, vol. 370, pp. 94--103, 2019.

\bibitem{LiangYuLi2025}
H.~Liang, J.~Yu, and H.~Li, ``Adaptive trajectory tracking control for small
  unmanned underwater vehicles with prescribed performance and dynamic
  compensation,'' \emph{IEEE Transactions on Industrial Electronics}, vol.~72,
  no.~6, pp. 6297--6306, 2025.

\bibitem{berger2024input}
T.~Berger, ``Input-constrained funnel control of nonlinear systems,''
  \emph{IEEE Transactions on Automatic Control}, vol.~69, no.~8, pp.
  5368--5382, 2024.

\bibitem{11121659}
A.~K. Gkesoulis and C.~P. Bechlioulis, ``A low-complexity adaptive performance
  control scheme for unknown nonlinear systems subject to input saturation,''
  \emph{IEEE Control Systems Letters}, vol.~9, pp. 2115--2120, 2025.

\bibitem{FotiadisRovithakisTACON2024}
F.~Fotiadis and G.~A. Rovithakis, ``Input-constrained prescribed performance
  control for high-order {MIMO} uncertain nonlinear systems via reference
  modification,'' \emph{IEEE Transactions on Automatic Control}, vol.~69,
  no.~5, pp. 3301--3308, 2024.

\bibitem{11187072}
A.~K. Gkesoulis, P.~A. Georgakis, G.~C. Karras, and C.~P. Bechlioulis,
  ``Prescribed performance control for uncertain {Euler-Lagrange} systems with
  constrained inputs via virtual-only reference modification,'' in \emph{2025
  European Control Conference (ECC)}, 2025, pp. 404--409.

\bibitem{gkesoulis2026mimocoupled}
A.~K. Gkesoulis, H.~E. Psillakis, G.~C. Karras, and C.~P. Bechlioulis,
  ``Prescribed performance control of uncertain {MIMO} nonlinear systems with
  coupled and constrained inputs,'' \emph{IEEE Control Systems Letters},
  vol.~10, pp. 97--102, 2026.

\bibitem{garrido2014automatic}
S.~Garrido-Jurado, R.~Mu{\~n}oz-Salinas, F.~J. Madrid-Cuevas, and M.~J.
  Mar{\'i}n-Jim{\'e}nez, ``{Automatic Generation and Detection of Highly
  Reliable Fiducial Markers Under Occlusion},'' \emph{Pattern Recognition},
  vol.~47, no.~6, pp. 2280--2292, 2014.

\bibitem{ubiquityrobotics_fiducials}
{Ubiquity Robotics}, ``{fiducials}: {ROS} packages for fiducial-based
  localization,'' \url{https://github.com/UbiquityRobotics/fiducials}, 2022,
  accessed: September 6, 2026.

\bibitem{moore2014generalized}
T.~Moore and D.~Stouch, ``{A Generalized Extended Kalman Filter Implementation
  for the Robot Operating System},'' in \emph{Proc. 13th Int. Conf. Intelligent
  Autonomous Systems (IAS-13)}, 2014, pp. 335--348.

\end{thebibliography}

\end{document}